\documentclass[a4paper,UKenglish,cleveref, autoref, thm-restate]{lipics-v2021}

\usepackage{breqn}
\usepackage{xspace}
\usepackage{tikz}
\usetikzlibrary{calc}
\usetikzlibrary{decorations}
\usetikzlibrary{patterns}
\usepackage{subcaption}
\usepackage{booktabs}
\usepackage{mathtools}
\usepackage{pdfrender}
\usepackage{needspace}

\usepackage [autostyle, english = american]{csquotes}
\MakeOuterQuote{"}

\title{Hardness of Multi-Agent Path Finding on Trees:\\A Unified Approach
} 

\titlerunning{Hardness of Multi-Agent Path Finding on Trees: A Unified Approach} %

\author{Tzvika Geft}{Rutgers University, NJ, USA %
}{%
}{0000-0002-3015-7514
}{}%

\authorrunning{T. Geft} %

\Copyright{Tzvika Geft} %

\ccsdesc[500]{Theory of computation~Problems, reductions and completeness}

\keywords{
Pebble Motion on Trees,
Coordinated Motion Planning,
Multi-Agent Path Finding,
Stack Rearrangement,
NP-Hardness} 

\category{} %

\relatedversion{} %

\nolinenumbers %

\EventEditors{Philip Bille, Seth Pettie, and Sabine Storandt}
\EventNoEds{3}
\EventLongTitle{34th Annual European Symposium on Algorithms (ESA 2026)}
\EventShortTitle{ESA 2026}
\EventAcronym{ESA}
\EventYear{2026}
\EventDate{August 31--September 4, 2026}
\EventLocation{L'Aquila, Italy}
\EventLogo{}
\SeriesVolume{388}
\ArticleNo{157}

\begin{document}

\maketitle

\begin{abstract}
This paper presents a simple framework that settles the complexity of Multi-Agent Path Finding (MAPF) on trees across standard objectives---distance, makespan, and flowtime---for both labeled and colored variants.
In MAPF, agents occupy the vertices of a graph and must move to target vertices without collisions while optimizing a given objective.
In the labeled case, the agents are distinct and have respective targets; in the colored case, agents of the same color are interchangeable.
While many MAPF variants are known to be intractable, several basic cases on trees have remained open.
We prove NP-hardness on trees for both labeled and 2-colored MAPF under all three objectives.
In particular, we resolve the classical Pebble Motion problem, where one pebble moves at a time to an adjacent empty vertex and the goal is to minimize the total number of moves.
Despite being  one of the most basic discrete motion models, its complexity on trees had remained open for several decades.
Moreover, for colored Pebble Motion, we give the first hardness result on any graph class, already with two colors, which is tight.

All of these results are established through the hardness of Stack Rearrangement, itself posed as an open problem, which asks to optimally rearrange items stored in stacks, and which we also prove to be NP-hard. Notably, the connection to stacks yields hardness already on very simple trees---subdivided stars---across all problems. Together, these results reveal a common tractability barrier that permeates several fundamental motion models, thereby unifying and strengthening prior hardness results. %

\end{abstract}

\newcommand{\tzvika}[1]{\textcolor{blue}{TG: #1}}

\newcommand{\wrap}[1]{\ensuremath{#1}\xspace}

\newcommand{\Npeb}{N}
\newcommand{\sol}{\ensuremath{\mathcal{S}}\xspace}

\newcommand{\sconf}{\ensuremath{S}\xspace}
\newcommand{\dconf}{\ensuremath{D}\xspace}
\newcommand{\spos}[1]{\wrap{\sconf(#1)}}
\newcommand{\dpos}[1]{\wrap{\dconf(#1)}}

\section{Introduction}

\emph{Pebble Motion on Graphs} (PMG), a generalization of the classical $(n^2-1)$-puzzle, asks for a shortest sequence of moves that takes a collection of labeled pebbles on the vertices of a graph from a start configuration to a target configuration, where a move relocates a single pebble to an adjacent empty vertex.
Its parallel-motion variant, \emph{Multi-Agent Path Finding} (MAPF), also known as \emph{Coordinated Motion Planning}, allows many pebbles---now called \emph{agents}---to move simultaneously,
and is typically cast as an optimization problem under objectives such as \emph{makespan}, the number of time steps until all agents reach their goals, and \emph{flowtime} (or \emph{sum-of-costs}), the sum of all the agents' individual arrival times.

In some cases, agents (or pebbles) can be naturally grouped into teams of equivalent agents rather than viewing each agent as unique.
In the \emph{$k$-colored} variant of PMG and MAPF~\cite{GoralyH10, coloredMAPF,DBLP:journals/ijrr/SoloveyH16}, agents are partitioned into $k$ teams, each with its own set of targets, and any assignment of agents within a team to its targets is valid provided all targets are eventually occupied.
The case where $k$ is the number of agents corresponds to the original labeled problem,
while the $k=1$ case is also known as \emph{unlabeled} or \emph{anonymous}.

Together, PMG, MAPF, and their variants serve as a prominent abstraction for multi-robot path planning
and have driven a large body of algorithmic work motivated by warehouse automation and related applications~\cite{DBLP:conf/socs/SternSFK0WLA0KB19,DBLP:journals/ai/SharonSFS15,DBLP:conf/aips/FelnerLB00KK18,BCP,LazyCBS}.

The \emph{feasibility} problem for PMG---finding any solution or reporting that none exists---is classically solvable in polynomial time for the labeled case~\cite {DBLP:conf/focs/KornhauserMS84}, as well as the colored case~\cite{GoralyH10}. %
The feasibility algorithms can produce solutions cubic in $n$, the number of vertices, while practical applications call for high quality solutions. 
In turn, the complexity of finding optimal solutions has received significant attention, with hardness results on general graphs, planar graphs, and 2D grids~\cite{Goldreich84,ratner1990n2,DBLP:conf/aaai/Surynek10,YuGraphs,DBLP:journals/ral/BanfiBA17,DBLP:journals/siamcomp/DemaineFKMS19,GH2021AAMAS,MAPF-socs23}. %
However, the picture on trees, the focus of this work, has not yet been settled.

Several recent works investigate motion on trees or tree-like graphs. %
Ardizzoni et al.~\cite{MAPF-PMT24} give an algorithm that finds feasible (not necessarily optimal) \emph{Pebble Motion on Trees} (PMT) solutions with $O(\Npeb n c + n^2)$ moves, for $n$ vertices, $\Npeb$ pebbles, and $c$ being the maximum corridor length in the tree, %
refining the classical $O(n^3)$ bound~\cite{DBLP:conf/focs/KornhauserMS84}.
Nakamigawa and Sakuma~\cite{NakamigawaS25} establish sharper upper bounds on the length of shortest PMG solutions, both on trees and on general graphs, in terms of graph diameter and cycle structure.
MAPF variants have been recently approached through parameterized complexity, where treewidth---a measure of how tree-like a graph is---has been used with other parameters to develop parameterized algorithms~\cite{MAPF-FPT2,FPT-DO-MAPF,DeligkasEGKR25}. %
On the hardness side,
Fioravantes et al.~\cite{MAPF-FPT2, FioravantesKKMOV25} prove NP-hardness of makespan-optimal MAPF on trees under various structural restrictions while 
Aichholzer et al.~\cite{AichholzerDKLLMRWW22} 
recently resolved the complexity of token swapping on trees---a structurally distinct setting where a move swaps two adjacent pebbles, which is not allowed in PMG or MAPF.

\subparagraph{Stack Rearrangement.}
One example of a motion planning problem that can be abstracted using PMT is
\emph{Stack Rearrangement} (SR).
The problem asks to find the shortest sequence of pop-and-push moves that reconfigures items distributed across a collection of bounded-capacity stacks to a target arrangement.
That is, each move removes an item from the top of one stack and places it at the top of another stack that has available capacity.
Han et al.~\cite{HanSBY18} introduce the problem to model robotic rearrangement in stack-like containers (e.g., stackable parking or retail shelves).
They observe that SR can be cast as a PMT problem and conjecture optimal SR to be intractable, but leave its complexity open.

\subsection*{Our Results}

We present a simple framework, rooted in the hardness of Stack Rearrangement, that settles the complexity of MAPF on trees across the standard objectives of total distance, makespan, and flowtime, for both labeled and colored variants (see \Cref{sec:defs} for formal definitions).
All variants are shown to be NP-hard on subdivided stars,\footnote{A \emph{subdivided star} is a graph consisting of any number of paths all joined at a single endpoint vertex.} among the simplest nontrivial trees, while all colored variants are shown to be NP-hard already for $k=2$.
To our knowledge, this is the first work to jointly establish the intractability of both labeled and colored MAPF across these three common objectives.
The results are summarized in \Cref{tab:comparison}.

\newcommand{\ck}{$\checkmark$}
\newcommand*{\bck}{%
  \textpdfrender{
    TextRenderingMode=FillStroke,
    LineWidth=.4pt, 
  }{\ck}%
}

\begin{table}[!ht]
\centering
\caption{Comparison to representative prior MAPF NP-hardness results that capture the state of the art.
Each entry specifies the optimization objective, the labeled vs. colored variant, and graph class.
Note that the table is not a complete survey and that works subsumed by the shown results at the graph type level are generally omitted.}
\label{tab:comparison}
\smallskip
\setlength{\tabcolsep}{4pt}
\renewcommand{\arraystretch}{1.12}
\begin{tabular}{llcccccc}
\toprule
\multirow{2}{*}{Graph type} & \multirow{2}{*}{Work} &
\multicolumn{3}{c}{Labeled} &
\multicolumn{3}{c}{Colored ($k=2$ unless specified)} \\
\cmidrule(lr){3-5} \cmidrule(lr){6-8}
& & Distance & Makespan & Flowtime
  & Distance & Makespan & Flowtime \\
\midrule
General      & \cite{YuGraphs}                     & \ck & \ck & \ck &     & \ck      & \ck \\
Planar       & \cite{YuPlanar}                     & \ck & \ck & \ck &     & \ck      & \ck \\
\multirow{2}{*}{2D subgrids}
             & \cite{GH2021AAMAS}                  & \ck &     &     &     &          &     \\
             & \cite{MAPF-socs23}                  &     & \ck & \ck &     & \ck      & \ck \\
Trees        & \cite{MAPF-FPT2,FioravantesKKMOV25} &     & \ck &     &     & $k=6$    &     \\
\midrule
\shortstack[l]{\textbf{Subdivided}\\\textbf{stars}}
             & \multirow{2}{*}{\shortstack[l]{\textbf{This}\\\textbf{work}}}
                                                          & \bck & \bck & \bck
                                                          & \bck & \bck & \bck
\\[3pt]
\shortstack[l]{\textbf{Max. degree}\\\textbf{3 trees}}
             &                                            & \bck &              & 
                                                          & \bck &              &              \\
\bottomrule
\end{tabular}
\end{table}

On trees, MAPF under the total distance objective, which minimizes total edge traversals, corresponds to the PMT optimization problem.
In this case the framework closes two particularly notable gaps:

\begin{itemize}
    \item \textbf{Pebble Motion on Trees (PMT)}. The complexity of optimal PMT has been implicitly open since the 1980s, when work on PMG emerged~\cite{DBLP:conf/focs/KornhauserMS84}.
    It was posed by Auletta et al.~\cite{AulettaMPP99}, which provides linear-time feasibility testing on trees, and recently re-raised by Deligkas et al.~\cite{FPT-DO-MAPF}.

    \item 
    \textbf{Colored PMG.}
    The complexity of optimal colored PMG has been open for \emph{any graph class} since at least Goraly and Hassin~\cite{GoralyH10}, who showed that feasibility is in P.
    Prior attempts, such as the systematic studies of the complexity of MAPF~\cite{YuGraphs,YuPlanar}, which consider
    standard objectives on general and later planar graphs, have left this case open.
    This work provides the first hardness result for colored PMG, already on trees and for $k=2$ (Theorem~\ref{thm:pmt2}).
\end{itemize}
\smallskip
Beyond subdivided stars, both the labeled and 2-colored PMT hardness results hold on comb-like trees of maximum degree~3 that naturally embed as 2D subgrids. %
The result for labeled PMT therefore recovers and strengthens the previously tightest 2D grid-based NP-hardness result~\cite{GH2021AAMAS}.

On the time-optimal side, we obtain the first NP-hardness result for flowtime on trees, as well as the first NP-hardness of 2-colored MAPF on trees for both makespan and flowtime, which was only known for more general graph classes.
Existing hardness results on trees are known only for the makespan objective, under various structural restrictions: 
maximum degree~5~\cite{MAPF-FPT2},
11 leaves~\cite{FioravantesKKMOV25},
or 9 internal vertices~\cite{FioravantesKKMOV25}.
Sharpening the makespan landscape, we simultaneously bound depth to~3 and the number of branching vertices, i.e., vertices of degree at least~3, to~1. %

For all considered objectives, the problems are efficiently solvable in the unlabeled (1-colored) case via a reduction to network flow~\cite{YuUnlabeledMAPF}.
Thus, the 2-colored hardness results form the exact tractability threshold.

\bigskip
\noindent\textbf{Stack rearrangement as a unifying source of hardness.}
All of the above results are established through reductions from Stack Rearrangement, whose own complexity we also settle:

\begin{itemize}
    \setlength{\topsep}{0pt}
    \item Labeled Stack Rearrangement is NP-hard, even for stacks of depth~3 (Theorem~\ref{thm:lsr}).

    \smallskip
    \item 2-colored Stack Rearrangement is NP-hard, even for deciding whether every item can be moved at most once (Theorem~\ref{thm:2lsr}).
\end{itemize}
\smallskip
Here too hardness for two colors is tight, as the unlabeled case is straightforward to solve efficiently.

The connection to stacks yields hardness on subdivided stars across MAPF variants: each stack becomes a path from a shared central vertex, and the serial bottleneck at that vertex tightly couples the stack and tree formulations.

As NP membership holds for all studied problems~\cite{DBLP:conf/focs/KornhauserMS84, GoralyH10, HanSBY18}, we only prove NP-hardness.

\subparagraph{Concurrent work.} During the review of this paper, we learned of independent work establishing hardness for labeled MAPF on trees under the total distance and flowtime objectives~\cite{KoyfmanEtAl2026}.
These concurrent results use different techniques and do not establish hardness on subdivided stars.

\section{Problem Definitions and Notation} \label{sec:defs}
Let $G=(V,E)$ be a connected undirected graph.
Let \(d(u,v)\) denote the shortest-path distance between vertices \(u\) and \(v\) in \(G\). %
We consider $\Npeb\le |V|$ uniquely labeled pebbles, also called agents, $p_1,\ldots,p_{\Npeb}$.
The pebbles' locations are specified by a \emph{configuration} $S$, which is an injective mapping assigning each pebble $p$ to a unique vertex $S(p)\in V$.

\begin{definition}[Pebble Motion on Graphs (PMG)]
Let $G = (V, E)$ be a connected undirected graph.
A \emph{move} consists of transferring a single pebble from its current vertex to an adjacent unoccupied vertex.
Given two configurations \sconf (start) and \dconf (target/goal), the PMG problem asks for a minimum-length sequence of moves that transforms \sconf into \dconf.
\end{definition}

\begin{definition}[Pebble Motion on Trees (PMT)]
A PMT instance is a PMG instance $I = (T, \sconf, \dconf)$ where $T$ is a tree.
\end{definition}

\begin{definition}[Multi-Agent Path Finding (MAPF)]\label{def:mapf}
A MAPF instance is given by a graph $G = (V, E)$ and two configurations $S$
and $D$ as in PMG.
Time is discretized and at each time step an agent can either wait at its current vertex or move to an adjacent vertex:
A \emph{timed path} for an agent $p$ is a sequence of vertices $(u_0, u_1, u_2, \ldots)$
with $u_0 = \spos{p}$, where each consecutive pair is either identical (a \emph{wait})
or adjacent in $G$.
A MAPF \emph{solution} is a set of timed paths, one per agent, such that every agent
eventually reaches and remains at its goal vertex, and the paths are
\emph{conflict-free}: no two agents coincide at a vertex at the same time step
(\emph{vertex conflict}), and no two agents traverse the same edge in opposite
directions at the same time step (\emph{edge conflict}).
In particular, an agent may enter a vertex $v$ in the same time step in which another agent leaves $v$, allowing agents to move in a train-like manner.

\subparagraph{Optimization objectives.}
The \emph{completion time} of an agent $p$ is the first time step $\tau_p$ at which
$p$ reaches $\dpos{p}$ and stays there.
The \emph{makespan} of a solution is $\max_p \tau_p$, i.e., the time it takes the last agent to reach its goal.
The \emph{flowtime} (also known as \emph{sum-of-costs}) of a solution is $\sum_p \tau_p$, i.e., the sum of agent completion times.
The \emph{total distance} of a solution is the total number of edge traversals over all timed paths.
\end{definition}

\subparagraph{Remark.} As MAPF with total distance objective on trees corresponds to PMT, we generally refer to the latter in the sequel.

We now introduce the stack-based formulation relevant to this work.

\begin{definition}[Labeled Stack Rearrangement (LSR)~\cite{HanSBY18}]
We are given a set of $\Npeb$ distinct objects, $\{o_1, o_2, \ldots, o_{\Npeb}\}$, and $w+1$ available stacks, each with a maximum capacity (depth) of $d$ such that $\Npeb \leq wd$.
In this problem, objects can be moved from the top of one stack to the top of another, via a \emph{pop-and-push action}, with the constraint that objects cannot be moved from an empty stack or to a full stack. An object's position in the stacks is defined by a 2-tuple, indicating its stack and its depth within that stack. 
A \emph{configuration} specifies the locations of all objects in the stacks.
Given an initial configuration, the objective of the LSR problem is to rearrange the objects to match a goal configuration using the minimum number of pop-and-push moves.
\end{definition}

We also consider the \emph{colored} variant of each of the labeled problems defined above.
In the colored variant, objects are partitioned into color classes and are interchangeable within each class.
Whenever the terms labeled or colored are not used, we default to the labeled case.

\begin{definition}[Colored variants: $k$-colored PMG, $k$-colored MAPF, $k$-colored Stack Rearrangement ($k$-colored SR)]
In the \emph{$k$-colored} variant, the objects (or pebbles/agents) are partitioned into $k$ color classes. The goal specifies, for each target position, only the required color, and any object of that color may occupy it.
\end{definition}

\section{Hardness of Stack Rearrangement}

We begin with the hardness of Stack Rearrangement variants, which will serve as base problems for all the other reductions.

\begin{theorem} \label{thm:lsr}
Labeled Stack Rearrangement (LSR) is NP-hard, even when restricted to stacks of depth 3.
\end{theorem}

\begin{proof}
To establish NP-hardness, we perform a reduction from the Minimum Feedback Arc Set (MFAS) problem, which asks to find the smallest set of edges in a directed graph whose removal results in a directed acyclic graph (DAG).
We use the fact that MFAS remains NP-hard for maximum degree 3~\cite{DBLP:books/fm/GareyJ79}; that is, we can assume the in-degree and out-degree of each vertex in $G$ are at most 2.

Given a directed graph $G = (V, E)$ with vertices $V = \{1, 2, \ldots, n\}$ and $m$ edges, we construct a corresponding LSR instance $I := I(G)$ as follows:
For each vertex $i$ in $G$, we introduce a corresponding object $o_i$.
For each edge $e=(i,j)$ in $G$, we create two objects denoted by $e_{ij}$ and $e'_{ij}$.
For each $o_i$, we designate a start stack where $o_i$ is initially at the bottom and a goal stack where it needs to end up at the bottom.
For an edge $(i,j)$, we place the start position of $e_{ij}$ in the start stack of $o_j$, anywhere above $o_j$, and set its goal to be at the bottom of its own stack. That is, the order of the start positions of the $e_{ij}$'s is arbitrary but they are all initially above $o_j$.
The start position of $e'_{ij}$ coincides with the goal of $e_{ij}$, and the goal of $e'_{ij}$ is in the goal stack of $o_i$, anywhere above $o_i$'s goal. See \Cref{fig:hardness}.
Note that an edge $(i,j)$ creates a dependency in an LSR solution with a lower bound cost (where every object is moved once), where $o_i$ must be moved to its goal before $o_j$. %

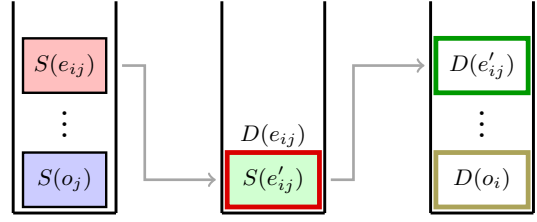
\begin{figure}
    \centering
    \begin{tikzpicture}[
    baseline=(current bounding box.south),
    stack/.style={line width=1.1pt},
    srect/.style={rectangle, draw=black, line width=0.9pt,
        minimum width=1.10cm, minimum height=0.72cm, inner sep=0pt, outer sep=0pt},
    trect/.style={rectangle, draw, line width=1.8pt,
        minimum width=1.22cm, minimum height=0.72cm, inner sep=0pt, outer sep=0pt, fill=white},
    mrect/.style={trect, draw=red!85!black, fill=green!18}
]
\def\stackW{1.36}
\def\cx{0.68} %
\def\stackH{2.80}
\def\gap{1.40}
\coordinate (L) at (0,0);
\coordinate (M) at (\stackW+\gap,0);
\coordinate (R) at (2*\stackW+2*\gap,0);
\foreach \X in {L,M,R}{
    \draw[stack] (\X) -- ++(0,\stackH);
    \draw[stack] ($(\X)+(\stackW,0)$) -- ++(0,\stackH);
    \draw[stack] (\X) -- ++(\stackW,0);
}
\def\ybot{0.44}
\def\ymid{1.28}
\def\ytop{1.95}
\node[srect, fill=red!25]  (lTop) at ($(L)+(\cx,\ytop)$) {\footnotesize $\spos{e_{ij}}$};
\node at ($(L)+(\cx,\ymid)$) {\LARGE$\vdots$};
\node[srect, fill=blue!20] (lBot) at ($(L)+(\cx,\ybot)$) {\footnotesize $\spos{o_j}$};
\node[mrect] (mBox) at ($(M)+(\cx,\ybot)$) {\footnotesize $\spos{e'_{ij}}$};
\node[anchor=south, inner sep=2pt] at (mBox.north) {\footnotesize $\dpos{e_{ij}}$};
\node[trect, draw=green!60!black]   (rTop) at ($(R)+(\cx,\ytop)$) {\footnotesize $\dpos{e'_{ij}}$};
\node at ($(R)+(\cx,\ymid)$) {\LARGE$\vdots$};
\node[trect, draw=yellow!60!black]  (rBot) at ($(R)+(\cx,\ybot)$) {\footnotesize $\dpos{o_i}$};
\draw[->, gray!70, line width=1.0pt]
    ($(L)+(\stackW+0.08, \ytop)$) --
    ++(0.30,0) --
    ++(0, \ybot-\ytop) --
    ($(M)+(-0.08, \ybot)$);
\draw[->, gray!70, line width=1.0pt]
    ($(M)+(\stackW+0.08, \ybot)$) --
    ++(0.30,0) --
    ++(0, \ytop-\ybot) --
    ($(R)+(-0.08, \ytop)$);
\end{tikzpicture}
    \vspace{-0.8cm}
    \caption{The dependency created by an edge $(i,j)$. (Left) Start stack for $o_j$; (middle) the stack for $(i,j)$; (right) the goal stack for $o_i$.}
    \label{fig:hardness}
\end{figure}

Finally, add $m$ empty stacks.
Therefore, the total number of objects is $n+2m$, and there are $2n+2m$ stacks.
Since every vertex in $G$ has in-degree and out-degree at most 2, a stack depth of 3 suffices.

We now prove the correctness of the reduction.
We show that $G$ has a feedback arc set $E' \subseteq E$ with $\ell$ edges if and only if $I$ has a solution with $n+2m+\ell$ moves.

Let $G'$ be the DAG obtained by removing the edges in $E'$ from $G$. Similarly, let $I'$ denote the instance obtained by removing all objects corresponding to $E'$ from $I$. We can use the topological ordering of $G'$ to solve $I'$ with exactly one move per object as follows.
For each vertex $i$ according to the topological ordering:
\begin{enumerate}
    \item For each incoming edge $(j, i)$, move $e_{ji}$. These objects are moved one by one per the ordering imposed by their start positions in the start stack of $o_i$.
    \item Move $o_i$.
    \item Then, for each outgoing edge $(i, j)$, move $e'_{ij}$ according to the ordering in the goal stack of $o_i$.
\end{enumerate}
The moves in step 1 are valid because when we handle vertex $i$, all the $e'_{ji}$'s have already been moved as part of handling step 3 in previous vertices. Moves in steps 2 and 3 are valid as no object blocks them.

We can now solve $I$, the original instance, based on the solution above.
In step 1, any $e_{ji}$ that corresponds to an edge in $E'$ is moved to an empty stack instead of its goal (thereby breaking the dependency of the edge).
Finally, after performing all the moves above, we can move all such $e_{ji}$'s to their goals one by one. This solution introduces $\ell$ additional moves beyond the lower bound for the number of moves, resulting in a total of $n+2m+\ell$ moves, as required.

Conversely, suppose we have a solution with $n+2m+\ell$ moves.
Let $E'$ denote the set of edges for which either $e_{ij}$ or $e'_{ij}$ moves more than once.
We have $|E'| \le \ell$.
We claim that $E'$ is a feedback arc set for $G$:
Order the vertices according to the order in which the corresponding objects make their
final moves.
We claim that this results in a topological ordering of $G'$, the graph where the edges of $E'$ are removed.
Indeed, let $(i,j)$ be an edge of $G'$.
Both $e_{ij}$ and $e'_{ij}$ move exactly once.
Since $e'_{ij}$ reaches its goal above $o_i$, it must move after the final move of $o_i$.
Moreover, $e'_{ij}$ must move before $e_{ij}$, which in turn must move before the final
move of $o_j$, since $e_{ij}$ initially lies above $o_j$.
Thus $o_i$ makes its final move before
$o_j$, i.e., vertex $i$ appears before $j$ in the ordering.
Hence the ordering is topological for $G'$, and $E'$ is a feedback arc set.
\end{proof}

\begin{theorem} \label{thm:2lsr}
2-colored SR is NP-hard, even for deciding whether there is a solution in which every item moves at most once.
\end{theorem}

\begin{proof}
We reduce from \textsc{3-Partition}: given $3m$ positive integers $a_1, \ldots, a_{3m}$ with each $a_i \in (K/4, K/2)$ and $\sum a_i = mK$, decide whether they can be partitioned into $m$ triples each summing to~$K$. This problem is strongly NP-complete~\cite{DBLP:books/fm/GareyJ79}.

\medskip\noindent\textbf{Construction.}
Given a \textsc{3-Partition} instance, we construct the following stacks, as shown in \Cref{fig:2colored-sr}.
We use red and blue as the two colors:
\begin{itemize}
    \item For each \(a_i\), create an \emph{item stack}, containing \(a_i\) red items
    initially and requiring \(a_i\) blue items in the target configuration.

    \item Create \(m-1\) \emph{checkpoint stacks}, each containing \(K\) blue items
    initially and requiring \(K\) red items in the target configuration.

    \item Create one additional \emph{red target stack}, initially empty and requiring
    \(K\) red items in the target configuration.

    \item Create one \emph{filler stack}, initially containing \(K\) blue items and empty
    in the target configuration.
\end{itemize}
The construction uses \(3m+(m-1)+2=4m+1\) stacks of depth \(K\).
It contains \(mK\) red items and \(mK\) blue items.
Since \textsc{3-Partition} is strongly NP-complete, the construction has polynomial size.

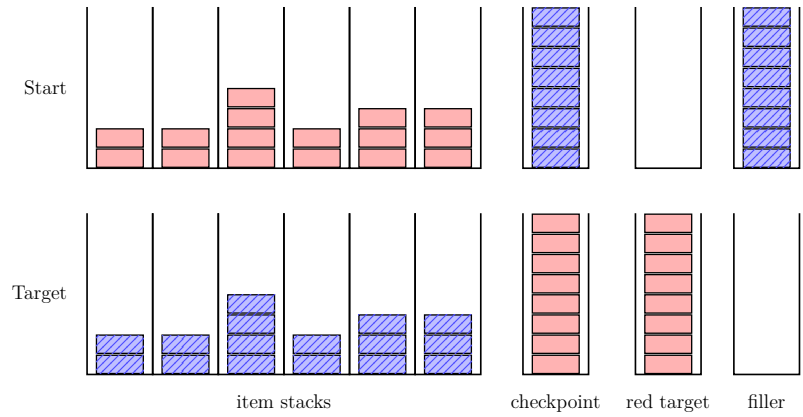
\begin{figure}[t]
\centering
\resizebox{0.75\columnwidth}{!}{%
\begin{tikzpicture}[
    box/.style={draw, thick, minimum width=1.0cm, minimum height=0.39cm},
    ritem/.style={fill=red!30},
    bitem/.style={fill=blue!25,
        postaction={pattern=north east lines, pattern color=blue!70}},
    stacklbl/.style={font=\Large, anchor=north}
]
\def\stackHW{0.7}
\def\boxH{0.43}
\def\wallCap{8}
\def\startBase{4.4}
\def\targetBase{0}

\newcommand{\drawStack}[4]{%
    \pgfmathsetmacro{\xL}{#1-\stackHW}%
    \pgfmathsetmacro{\xR}{#1+\stackHW}%
    \pgfmathsetmacro{\wallH}{\wallCap*\boxH}%
    \draw[line width=1pt] (\xL,{#2+\wallH}) -- (\xL,{#2}) -- (\xR,{#2}) -- (\xR,{#2+\wallH});%
    \pgfmathtruncatemacro{\ni}{#3}%
    \ifnum\ni>0%
        \foreach \r in {1,...,\ni} {%
            \pgfmathsetmacro{\by}{#2 + (\r-0.5)*\boxH}%
            \node[box, #4] at (#1, \by) {};%
        }%
    \fi%
}

\def\xA{0.7}   \def\xB{2.1}   \def\xC{3.5}
\def\xD{4.9}   \def\xE{6.3}   \def\xF{7.7}
\def\xG{10.0}  \def\xH{12.4}  \def\xI{14.5}

\drawStack{\xA}{\startBase}{2}{ritem}
\drawStack{\xB}{\startBase}{2}{ritem}
\drawStack{\xC}{\startBase}{4}{ritem}
\drawStack{\xD}{\startBase}{2}{ritem}
\drawStack{\xE}{\startBase}{3}{ritem}
\drawStack{\xF}{\startBase}{3}{ritem}
\drawStack{\xG}{\startBase}{8}{bitem}
\drawStack{\xH}{\startBase}{0}{ritem}
\drawStack{\xI}{\startBase}{8}{bitem}

\drawStack{\xA}{\targetBase}{2}{bitem}
\drawStack{\xB}{\targetBase}{2}{bitem}
\drawStack{\xC}{\targetBase}{4}{bitem}
\drawStack{\xD}{\targetBase}{2}{bitem}
\drawStack{\xE}{\targetBase}{3}{bitem}
\drawStack{\xF}{\targetBase}{3}{bitem}
\drawStack{\xG}{\targetBase}{8}{ritem}
\drawStack{\xH}{\targetBase}{8}{ritem}
\drawStack{\xI}{\targetBase}{0}{ritem}

\node[font=\Large, anchor=east] at (-0.3, {\startBase + 4*\boxH}) {Start};
\node[font=\Large, anchor=east] at (-0.3, {\targetBase + 4*\boxH}) {Target};

\node[stacklbl] at ({(\xA+\xF)/2}, -0.3) {item stacks};
\node[stacklbl] at (\xG, -0.3) {checkpoint};
\node[stacklbl] at (\xH, -0.3) {red target};
\node[stacklbl] at (\xI, -0.3) {filler};

\end{tikzpicture}%
}
\caption{The 2-colored SR construction for the \textsc{3-Partition} instance with
$K=8$ and items $(2,2,4,2,3,3)$: \sconf with \dconf below. Blue items are shown with a hatching pattern.}
\label{fig:2colored-sr}
\end{figure}

\smallskip\noindent\textbf{Structural observation.}
We say that a stack is \emph{cleared} when
every item originally in that stack has moved out of it.
In any solution where each item moves at most once, every item stack and every checkpoint stack must be cleared before any item is placed in it.
Indeed, any incoming item placed in such a stack earlier would permanently trap wrong-colored items that need to be removed.
We therefore track available red and blue target capacity using two counters: let $r$ (resp. $b$) denote the number of currently accessible red (resp. blue) target positions.
Initially $r = K$ (due to the empty red target stack) and $b = 0$.
Clearing item stack~$i$ decreases $r$ by $a_i$ and increases $b$ by $a_i$; clearing a checkpoint stack decreases $b$ by $K$ and increases $r$ by $K$.
At any point in a solution we have $r \geq 0$ and $b \geq 0$.
 
\smallskip\noindent\textbf{Correctness}.
Given a valid 3-partition $S_1, \ldots, S_m$, we construct a one-move-per-item SR solution that operates in $m$ phases:
For $j = 1, \ldots, m$, clear item stacks corresponding to $S_j$, then clear a checkpoint stack, except for $j=m$ where the filler stack is cleared instead.
At the start of each phase $r = K, b = 0$ and one target stack is available for red items: the red target stack for $j=1$, or a checkpoint stack otherwise.
Clearing the $S_j$ stacks brings us to $r = 0, b = K$ as the cleared stacks can now accept blue items.
Clearing the checkpoint resets the state to $r = K, b = 0$.
 
Suppose a valid 2-colored SR solution exists in which every item is moved once.
Let $Q_1, \ldots, Q_{m-1}$ denote the checkpoint stacks in the order in which they are cleared.
Let \(S_1\) be the set of item stacks
cleared before the first checkpoint stack $Q_1$ is cleared, and let $q=\sum_{i\in S_1} a_i$.
Before $Q_1$ is cleared, there are exactly $K$ red target positions available, which are in the initially empty red stack, so \(q\le K\).
Since the
\(K\) blue items originating in $Q_1$ each move exactly once, they must
occupy \(K\) blue target positions in item stacks cleared before this time.
Hence
\(q\ge K\).
Therefore \(q=K\), and since each \(a_i\in(K/4,K/2)\), we have
\(|S_1|=3\) and so $S_1$ corresponds to a valid triple.
Furthermore,
 \(q=K\)
 implies that when $Q_1$ becomes cleared, every checkpoint and item stack other than \(S_1\cup\{Q_1\}\) remains as in the initial configuration.
The red target-only stack is filled to its capacity and $Q_1$ is now empty, and so the initial target capacity state $r = K, b = 0$ is restored.

The same argument repeats for each subsequent \(Q_j\): the item stacks cleared after
\(Q_{j-1}\) is cleared and before \(Q_j\) is cleared correspond to a triple summing to \(K\).
Finally, the number of items in the last 3 item stacks,
cleared after \(Q_{m-1}\) is cleared, must be \(K\), and so we obtain $m$ triples forming a valid 3-partition.
\end{proof}

\section{Hardness of distance-optimal motion}
This section reduces Stack Rearrangement to PMT variants and establishes the hardness of the latter.

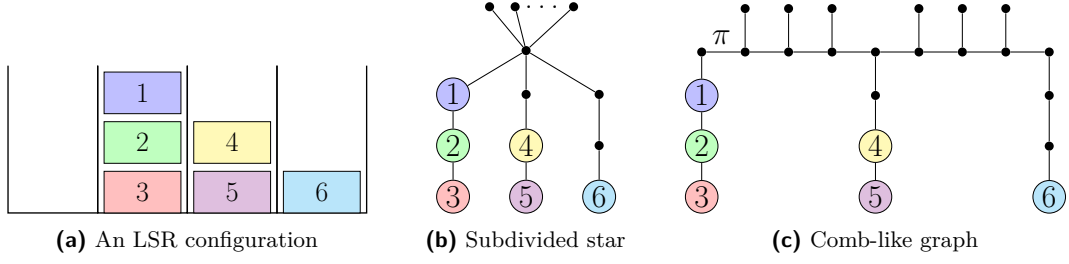
\begin{figure}[t]
\centering
\captionsetup[subfigure]{justification=centering}

\tikzset{
    box/.style={draw, thick, minimum width=2.4cm, minimum height=1.3cm, font=\Huge},
    circnode/.style={draw, thick, circle, minimum size=0.9cm, font=\Huge, inner sep=0pt},
    dotnode/.style={fill=black, circle, inner sep=2.5pt}
}

\subcaptionbox{An LSR configuration\label{fig:lsr}}[0.34\columnwidth]{%
\centering
\resizebox{0.34\columnwidth}{!}{%
\begin{tikzpicture}[every node/.style={font=\Large},
    baseline=(current bounding box.south)]
    \def\numCols{4}
    \def\colW{2.8}
    \def\colH{4.6}
    \foreach \i in {0,1,2,3} {
        \pgfmathsetmacro{\xL}{\i*\colW}
        \pgfmathsetmacro{\xR}{\xL+\colW}
        \draw[line width=1.2pt] (\xL,\colH) -- (\xL,0) -- (\xR,0) -- (\xR,\colH);
    }
    \foreach \col/\label/\clr/\by in {
        1/3/red!25/0.65,
        1/2/green!25/2.20,
        1/1/blue!25/3.75,
        2/5/violet!25/0.65,
        2/4/yellow!35/2.20,
        3/6/cyan!25/0.65%
    } {
        \pgfmathsetmacro{\bx}{\col*\colW + \colW/2}
        \node[box, fill=\clr] at (\bx, \by) {\label};
    }
\end{tikzpicture}%
}%
}%
\hfill
\subcaptionbox{Subdivided star\label{fig:star}}[0.25\columnwidth]{%
\centering
\resizebox{0.17\columnwidth}{!}{%
\begin{tikzpicture}[every node/.style={font=\Large},
    baseline=(root.center)]
    \def\yRoot{5.2}
    \def\yA{4.0}
    \def\yB{2.6}
    \def\yC{1.2}
    \def\yBuf{6.4}
    \def\branchSep{2.0}
    \pgfmathsetmacro{\rootX}{\branchSep * (3-1)/2 + \branchSep/2}
    \pgfmathsetmacro{\bI}{\branchSep/2 + 0*\branchSep}
    \pgfmathsetmacro{\bII}{\branchSep/2 + 1*\branchSep}
    \pgfmathsetmacro{\bIII}{\branchSep/2 + 2*\branchSep}

    \node[dotnode] (root) at (\rootX, \yRoot) {};

    \node[dotnode] (buf1) at (\rootX - 1.0, \yBuf) {};
    \node[dotnode] (buf2) at (\rootX - 0.3, \yBuf) {};
    \node[dotnode] (buf3) at (\rootX + 1.3, \yBuf) {};
    \draw[thick] (root) -- (buf1);
    \draw[thick] (root) -- (buf2);
    \draw[thick] (root) -- (buf3);
    \node at (\rootX + 0.5, \yBuf) {\Huge\ldots};

    \node[circnode, fill=blue!25]   (n1) at (\bI, \yA) {1};
    \node[circnode, fill=green!25]  (n2) at (\bI, \yB) {2};
    \node[circnode, fill=red!25]    (n3) at (\bI, \yC) {3};

    \node[dotnode]                  (d2) at (\bII, \yA) {};
    \node[circnode, fill=yellow!35] (n4) at (\bII, \yB) {4};
    \node[circnode, fill=violet!25] (n5) at (\bII, \yC) {5};

    \node[dotnode]                  (d3a) at (\bIII, \yA) {};
    \node[dotnode]                  (d3b) at (\bIII, \yB) {};
    \node[circnode, fill=cyan!25]   (n6)  at (\bIII, \yC) {6};

    \draw[thick] (root) -- (n1);
    \draw[thick] (root) -- (d2);
    \draw[thick] (root) -- (d3a);
    \draw[thick] (n1) -- (n2);
    \draw[thick] (n2) -- (n3);
    \draw[thick] (d2) -- (n4);
    \draw[thick] (n4) -- (n5);
    \draw[thick] (d3a) -- (d3b);
    \draw[thick] (d3b) -- (n6);
\end{tikzpicture}%
}%
}%
\hfill
\subcaptionbox{Comb-like graph\label{fig:comb}}[0.36\columnwidth]{%
\centering
\resizebox{0.36\columnwidth}{!}{%
\begin{tikzpicture}[every node/.style={font=\Large},
    baseline=(sa1.center)]
    \def\yPi{5.2}
    \def\yA{4.0}
    \def\yB{2.6}
    \def\yC{1.2}
    \def\yBuf{6.4}
    \def\piSep{1.2}

    \node[dotnode] (sa1) at (0, \yPi) {};
    \node[dotnode] (cb1) at (\piSep, \yPi) {};
    \node[dotnode] (cb2) at (2*\piSep, \yPi) {};
    \node[dotnode] (cb3) at (3*\piSep, \yPi) {};
    \node[dotnode] (sa2) at (4*\piSep, \yPi) {};
    \node[dotnode] (cb4) at (5*\piSep, \yPi) {};
    \node[dotnode] (cb5) at (6*\piSep, \yPi) {};
    \node[dotnode] (cb6) at (7*\piSep, \yPi) {};
    \node[dotnode] (sa3) at (8*\piSep, \yPi) {};

    \draw[thick] (sa1) -- (cb1) -- (cb2) -- (cb3) -- (sa2)
                       -- (cb4) -- (cb5) -- (cb6) -- (sa3);

    \node[circnode, fill=blue!25]   (c1) at (0, \yA) {1};
    \node[circnode, fill=green!25]  (c2) at (0, \yB) {2};
    \node[circnode, fill=red!25]    (c3) at (0, \yC) {3};
    \draw[thick] (sa1) -- (c1) -- (c2) -- (c3);

    \node[dotnode]                  (c2top) at (4*\piSep, \yA) {};
    \node[circnode, fill=yellow!35] (c4) at (4*\piSep, \yB) {4};
    \node[circnode, fill=violet!25] (c5) at (4*\piSep, \yC) {5};
    \draw[thick] (sa2) -- (c2top) -- (c4) -- (c5);

    \node[dotnode]                  (c3top) at (8*\piSep, \yA) {};
    \node[dotnode]                  (c3mid) at (8*\piSep, \yB) {};
    \node[circnode, fill=cyan!25]   (c6) at (8*\piSep, \yC) {6};
    \draw[thick] (sa3) -- (c3top) -- (c3mid) -- (c6);

    \node[dotnode] (lf1) at (\piSep, \yBuf) {};
    \draw[thick] (cb1) -- (lf1);
    \node[dotnode] (lf2) at (2*\piSep, \yBuf) {};
    \draw[thick] (cb2) -- (lf2);
    \node[dotnode] (lf3) at (3*\piSep, \yBuf) {};
    \draw[thick] (cb3) -- (lf3);
    \node[dotnode] (lf4) at (5*\piSep, \yBuf) {};
    \draw[thick] (cb4) -- (lf4);
    \node[dotnode] (lf5) at (6*\piSep, \yBuf) {};
    \draw[thick] (cb5) -- (lf5);
    \node[dotnode] (lf6) at (7*\piSep, \yBuf) {};
    \draw[thick] (cb6) -- (lf6);

    \node[font=\Huge, above right, xshift=2pt] at (sa1.north east) {$\pi$};
\end{tikzpicture}%
}%
}%

\caption{The reductions in \Cref{thm:pmt}. 
A 4-stack LSR configuration (a) and the corresponding configurations in the subdivided-star (b) and comb-like graph (c) PMT instances.
The buffer leaves are the topmost vertices.}
\label{fig:labeled-pmts}
\end{figure}

\begin{theorem}  \label{thm:pmt}
PMT is NP-hard, even when restricted to (i) subdivided stars of depth 3
or to (ii) trees of maximum degree 3.
\end{theorem}
\begin{proof}
We perform reductions from LSR with a stack depth of 3, which is NP-hard by \Cref{thm:lsr}.
See \Cref{fig:labeled-pmts}.

\smallskip
For (i), an LSR instance is converted to a subdivided star by treating each stack as a path with length equal to the stack's depth and then joining the top vertices of the stacks to a single root vertex $c$.
Each PMT vertex on a stack path corresponds to one LSR stack position.
By \Cref{thm:lsr} we can assume that there are as many buffer (initially empty) stacks as objects, so we add $N$ empty buffer stacks, which are leaves emanating from $c$.

Let $L$ denote the sum over all pebbles of the shortest-path distance from the pebble's start to its goal. Then $L$ is a lower bound on the number of moves of any PMT solution. We claim that the LSR instance admits a solution with at most $\ell$ relocations, i.e., moves to a temporary stack, if and only if the PMT instance admits a solution with at most $L + 2\ell$ moves.

Given an LSR solution with $\ell$ relocations, each non-relocated object corresponds to a pebble moved along its shortest path to its goal.
Each of the $\ell$ relocated objects corresponds to a pebble that first detours into an (unoccupied) buffer leaf, adding exactly two moves, before proceeding to its goal. %
The ample buffers guarantee that a buffer is free at the time of each detour.
The PMT solution has a total of $L + 2\ell$ moves.

Conversely, suppose we have a PMT solution with $L + 2\ell$ moves.
Each move of a pebble from the center vertex $c$
to a stack path corresponds to a single pop-and-push LSR move.
Since using a buffer in the PMT instance requires a detour, adding two moves to the motion of a pebble, the resulting LSR solution has at most $\ell$ relocations.

\smallskip
For (ii), we construct a comb-like graph. We use a central path $\pi$ and for each stack, we attach its corresponding stack path to a distinct vertex of $\pi$.
For each pebble, we provide a unique empty leaf vertex that it can detour to and incur exactly two extra moves beyond its shortest path. 
Specifically, for each pebble $p$, let $u_1$ and $u_2$ be the unique vertices of $\pi$ to which the start and target stack path of $p$, respectively, are attached.
Subdivide an arbitrary edge on the subpath of $\pi$ between $u_1$ and $u_2$ by a new vertex $b$, and attach a new leaf to $b$.
The construction adds $O(N)$ vertices and has a maximum degree $3$. %

As in (i), we claim that the LSR instance admits a solution with at most $\ell$ relocations if and only if the PMT instance admits a solution with at most $L + 2\ell$ moves. %

It is straightforward to verify that the forward direction is analogous to (i), so we prove only the backward direction.
This direction requires more careful analysis since the PMT solution can now result in multiple pebbles that are simultaneously located outside of a stack path.

Suppose we have a PMT solution on the comb graph with $L + 2\ell$ moves. We show it yields a valid LSR solution with at most $\ell$ relocations.
Since each detour beyond a shortest path requires at least $2$ extra moves, at most $\ell$ pebbles follow a non-shortest path to their goal.
Call these pebbles \emph{detoured} and the remaining pebbles \emph{direct}.
We present an LSR solution for the direct pebbles that does not involve relocations.
Each detoured pebble can be easily incorporated into this LSR solution as an object $o$ that is relocated exactly once, as a post-processing step:
Each such $o$ is moved to an empty stack (not shared with any other object) and then to its target position when it becomes available, without disturbing other moves.
This accounts for at most $\ell$ relocations.

It remains to show that the direct pebbles yield a valid LSR solution without relocations.
We remove all detoured pebbles and their moves from the PMT solution, and call the resulting solution \sol.

\begin{lemma} \label{lem:atomicity} %
Let $p$ be a pebble whose motion in \sol follows a shortest path $v_1, v_2, \ldots, v_r$ in the comb graph from its start vertex to its goal vertex.
Then moves in the solution \sol can be rearranged (without changing the total move count) so that $p$ moves from $v_1$ to $v_r$ in $r-1$ consecutive steps, i.e., with no other pebble moving during this interval.
We call such a traversal \emph{atomic}.
\end{lemma}

\begin{proof}
We use induction on $r$. The base case $r = 2$ is trivial. For the inductive step, assume that there is a solution \sol' where $p$'s traversal of $v_1, \ldots, v_q$ is atomic.
In $\sol'$, some sequence of moves $M$ by other pebbles occurs between $p$'s arrival at $v_q$ and $p$'s move to $v_{q+1}$. Since $p$ occupies $v_q$ throughout $M$, no move in $M$ crosses $v_q$, so each move lies entirely in one component of $T \setminus \{v_q\}$.
This yields a partition of $M$: $M_{\text{next}}$, the moves in the component containing $v_{q+1}$, and $M_{\text{rest}}$, the moves in the remaining component(s).

We modify the solution as follows: perform the moves of $M_{\text{next}}$ right before $p$'s atomic block, then $p$ traverses $v_1, \ldots, v_{q+1}$ atomically, then perform the moves of $M_{\text{rest}}$. Subsequent moves remain the same.

This reordering is valid: First, since $p$ follows a shortest path in a tree, $v_{q+1} \notin \{v_1, \ldots, v_q\}$, so the component of $T \setminus \{v_q\}$ containing $v_{q+1}$ does not contain $v_1, \ldots, v_{q-1}$. The moves of $M_{\text{next}}$ and $p$'s atomic traversal of $v_1, \ldots, v_q$ thus act on disjoint vertex sets, so $M_{\text{next}}$ can be performed before $p$'s atomic block without affecting validity. Second, the moves of $M_{\text{rest}}$ lie in components of $T \setminus \{v_q\}$ not containing $v_{q+1}$, so they do not interact with $p$'s move from $v_q$ to $v_{q+1}$ and can be deferred.
\end{proof}

Apply the lemma to each pebble in turn.
Processing pebbles in sequence does not invalidate previously established atomic blocks, as any such block within $M$ lies entirely in one component of $T \setminus \{v_q\}$ and is shifted as a unit and never split.
The result is a solution consisting of atomic start-to-target traversals, each corresponding to a pop-and-push move in LSR, as required.
\end{proof}

\Needspace{4\baselineskip}
\begin{theorem} \label{thm:pmt2}
2-colored PMT is NP-hard, even when restricted to (i) subdivided stars or to (ii) trees of maximum degree 3.
\end{theorem}
\begin{proof}
We reduce from 2-colored SR, which is NP-hard by \Cref{thm:2lsr}.

\medskip
For (i), apply the corresponding reduction used in the proof of \Cref{thm:pmt} to yield a subdivided star.
By \Cref{thm:2lsr}, it is NP-hard to decide whether a 2-colored SR instance can be solved without relocations.
Since a zero relocation SR solution corresponds to a PMT solution with cost exactly $L$ (the shortest-path lower bound), and vice versa,
it is NP-hard to decide if a 2-colored PMT instance admits a solution of cost $L$.
Notice that any assignment of pebbles to targets always results in the same distance lower bound of $L$.

\medskip
For (ii), we apply the comb graph reduction used in the proof of \Cref{thm:pmt}, i.e., take $T$ to be a central path $\pi$ to which all the stack paths are connected.
The stacks are ordered along $\pi$ as follows: all $3m$ item stacks on the left, then all $m-1$ checkpoint stacks, the empty red target stack, and the blue filler stack on the right.
No additional empty stacks are used.
It is straightforward to verify that this arrangement of the stacks results in the same distance lower bound $L$ for any valid assignment of pebbles to targets.
A zero-relocation SR solution corresponds to a PMT solution of cost $L$.
Conversely, a cost-$L$ PMT solution forces every pebble to move via a shortest path, and so we apply \Cref{lem:atomicity} to make each pebble's motion atomic to yield a zero-relocation SR solution.
\Cref{lem:atomicity} is valid for the 2-colored case because a solution fixes a unique target and path for each pebble.
\end{proof}

\section{Hardness of time-optimal objectives}
In this section, we prove NP-hardness for time-optimal MAPF objectives, namely makespan and flowtime.

\begin{theorem} \label{thm:mapf-make}
  MAPF is NP-hard for the makespan objective, even when restricted to a tree that is a subdivided star of depth 3.
\end{theorem}

\begin{proof}
We reduce from LSR with stack depth 3, which is NP-hard by Theorem~\ref{thm:lsr}.
  Given an LSR instance $I$ with $\Npeb$ items, we construct a subdivided star $T$ as in \Cref{thm:pmt}(i):
  each stack of depth $d$ becomes a path $v_1, \ldots, v_d$, and the vertices $v_1$ of all paths are joined to a single central vertex $c$. In the initial configuration, the items of each stack containing $a$ items are placed at vertices $v_1, \ldots, v_a$, packed toward $c$. The target configuration is defined analogously from the target stack configuration: if a stack contains $a$ items in the target configuration, then those items occupy $v_1, \ldots, v_a$, again packed toward $c$.

  We claim that the LSR instance $I$ has a solution with at most $M$ moves if and only if the resulting MAPF instance on $T$ admits a schedule with makespan at most $M + 1$.

  Suppose $I$ has a solution with $M$ moves. Each SR move, which pops the top item of one stack and pushes it onto another stack, is realized in $T$ by moving the corresponding agent through $c$. We schedule these $M$ transits in consecutive time steps, so that $c$ is occupied at times $2,\ldots,M+1$.
  More precisely, at each intermediate time step, the agent currently at $c$ moves to the destination stack of the current SR move, while the agent for the next SR move enters $c$ from the source stack of that move. Within each source stack path, when the agent at $v_1$ moves to $c$, the agents at $v_2, v_3, \ldots$ shift one step toward $c$ in the same time step, so that $v_1$ is always occupied for a non-empty stack.
  Similarly, when an agent moves from $c$ into a destination stack path, the existing agents in that path shift one step away from $c$, so the arriving agent can occupy $v_1$. 
  Thus the first transit enters $c$ at time $1$, and the last transit leaves $c$ at time $M+1$. Hence the MAPF schedule has makespan $M+1$.

  Suppose the MAPF instance on $T$ admits a solution with makespan $M+1$. Since $c$ is empty initially, there are at most $M$ time steps at which an agent located at $c$ goes to a stack path.
  Each such time step corresponds to a pop-and-push move of an item in $I$, hence these time steps define at most $M$ SR moves.
  Along each stack path, agents cannot pass one another, so these induced moves respect the LIFO order.
  Therefore, we obtain a valid LSR solution with at most $M$ moves.
\end{proof}

\begin{theorem} \label{thm:mapf-make-2}
  2-colored MAPF is NP-hard for the makespan objective, even when restricted to a tree that is a subdivided star.
\end{theorem}

\begin{proof}
  We reduce from 2-colored SR, which is NP-hard by \Cref{thm:2lsr}, exactly as in \Cref{thm:mapf-make}.
  Since we can track moves of the colored case the same way as in the labeled case, the conversions between solutions are identical. %
\end{proof}

We turn to the flowtime objective, where we begin with hardness for the 2-colored case, where the argument is simpler because the reduction only needs to distinguish zero SR relocations from any relocations.
In the labeled case, on the other hand, the reduction must more carefully control the relationship between flowtime and the number of relocations.

\begin{theorem} \label{thm:flow2}
2-colored MAPF is NP-hard for the flowtime objective, even when restricted to a tree that is a subdivided star.
\end{theorem}

\begin{proof}
We reduce from 2-colored SR, which is NP-hard by \Cref{thm:2lsr}.
From the construction in \Cref{thm:2lsr}, we may assume that every item starts in a non-goal stack and that it is NP-hard to decide whether an instance with $\Npeb$ items admits a solution using exactly $\Npeb$ moves.

Given such a 2-colored SR instance $I$ with $\Npeb$ items, we construct a subdivided star $T$ with central vertex $c$, together with start and target configurations, exactly as in Theorem~\ref{thm:mapf-make}. %

Set $\Delta \coloneqq \sum_{v \in D} d(c,v)$,
where \(D\) denotes the set of target vertices in $T$.
We claim that $I$ admits a solution with exactly $\Npeb$ moves if and only if the resulting MAPF instance on $T$ admits a solution with flowtime of $\Npeb(\Npeb+1)/2 + \Delta$.

Fix an optimal MAPF solution.
For an agent \(a\), let \(t_a\) denote the timestep at which \(a\) reaches \(c\) for the last time, and let \(\delta_a\) be the distance from \(c\) to the target vertex eventually occupied by \(a\).
After traversing $c$ for the last time, the shortest possible remaining path for $a$ is a descent of length $\delta_a$ into its target stack, so the completion time of $a$ is $t_a + \delta_a$.
Hence, the flowtime satisfies
$ F = \sum_a (t_a + \delta_a) \;=\; \sum_a t_a + \Delta
$.
As the values $\{t_a\}_a$ are distinct positive integers, $F$ is minimized with $F= \Npeb(\Npeb+1)/2 + \Delta$ if and only if $\{t_a\}_a = \{1, 2, \ldots, \Npeb\}$.

Suppose $I$ admits a solution with exactly $\Npeb$ moves.
Each SR move is converted to a single transit of the corresponding agent through $c$, scheduling the $\Npeb$ transits to reach $c$ at consecutive timesteps $1, 2, \ldots, \Npeb$.
This is done by having all agents in a stack move up simultaneously as the topmost agent leaves.
Each agent $a$ descends directly into its target vertex without waiting after visiting $c$, resulting in a flowtime of $\Npeb(\Npeb+1)/2 + \Delta$.

Suppose the MAPF instance on $T$ admits a solution with flowtime at most ${\Npeb(\Npeb+1)/2 + \Delta}$.
That is, we have $\{t_a\}_a = \{1, \ldots, \Npeb\}$.
The equality is only possible if each agent visits $c$ once, which corresponds to exactly one valid move per object in the SR solution.
\end{proof}

For the labeled case under the flowtime objective, repeating the reduction above would not work since the order in which relocations occur affects flowtime: earlier relocations correspond to increased agent completion times.
The following reduction normalizes such variations by adding many agents that would be delayed by relocations  regardless of when they occur in an LSR solution.

\begin{theorem} \label{thm:flowtime}
MAPF is NP-hard for the flowtime objective, even when restricted to a tree that is a subdivided star.
\end{theorem}

\begin{proof}
We reduce from LSR, which is NP-hard by \Cref{thm:lsr}.
Given an LSR instance $I$ with a relocation bound $\ell$, we construct a subdivided star $T$ with center vertex $c$ as in \Cref{thm:pmt}(i):
each stack is represented by a path joined to $c$.

We modify $T$ as follows for each stack path $\sigma$:
Replace the edge connecting $\sigma$ to $c$ by a path containing $\ell+1$ new vertices $b^1_\sigma,\dots,b^{\ell+1}_\sigma$. 
Introduce $\ell+1$ \emph{seal agents} $s^1_\sigma,\dots,s^{\ell+1}_\sigma$, with target vertices $b^1_\sigma,\dots,b^{\ell+1}_\sigma$, respectively.
Introduce a \emph{gate agent} $g_\sigma$.
Initially the seal agents are located in a new stack path $H_\sigma$ of length $\ell+2$, above $g_\sigma$.

Finally, introduce $K$ (whose value is specified later) \emph{padding agents} in a new start stack path, ordered from top to bottom as $q_1,\dots,q_K$.
The gate agents and padding agents have a common new target stack $Z$, where the gate agents must occupy the deepest positions, in an arbitrary fixed order, and the padding agents occupy the positions above them, ordered by $q_1,\dots,q_K$ from bottom to top.

Let $N'$ be the total number of agents in $T$.
Let $F_0 \coloneqq N'(N'+1)/2+\Delta$ be the flowtime lower bound, corresponding to a solution with no relocations, as defined in the proof of \Cref{thm:flow2}.
We claim that $I$ admits a solution with at most $\ell$ relocations if and only if the MAPF instance on $T$ admits a solution with flowtime at most $B \coloneqq F_0+\ell N'$.

Suppose that $I$ admits a solution with at most $\ell$ relocations.
We simulate each LSR move by a corresponding transit through $c$ as in the 2-colored case.
That is, $c$ is occupied at every timestep except the last and after visiting $c$ for the last time, each agent descends into its target without waiting.
After all original LSR moves
have been simulated, move the seal agents ("sealing" the original stacks) and then the gate agents directly to their respective targets.
Finally, move the $K$ padding agents directly to their targets.
The LSR relocations correspond to at most $\ell$ \emph{extra transits} through $c$, i.e., an agent passing through $c$ but not for the last time.
Each such transit delays at most all $N'$ final transits by one time step.
Hence, the MAPF solution has a flowtime $F \le F_0+\ell N'=B$, as required.

Conversely, suppose there is a MAPF solution with a flowtime $F \le B$.
Let $\tau$ be the last timestep at which a gate agent transits through $c$.
Each padding agent's final transit through $c$ occurs after $\tau$, since the gate agents need to occupy the deepest vertices in $Z$.
Hence, each extra transit through $c$ occurring before $\tau$ precedes the $K$ padding agents' final transits.
Let $r$ be the number of extra transits before $\tau$.
We have $F \ge F_0 + rK$ since at least $K$ agents' final transits are delayed by $r$ timesteps compared to the lower bound.
Let $m$ denote the number of non-padding agents, which is fixed for a given $I$.
Choose $K$ such that $K > \ell m$.
If $r \ge \ell + 1$, then $rK \ge (\ell+1)K > \ell K + \ell m = \ell N'$, so $F > F_0 + \ell N' = B$, which is a contradiction. Thus $r \le \ell$.

Fix a stack path $\sigma$.
Each seal agent $s^t_\sigma$ lies above $g_\sigma$ in $H_\sigma$, so it leaves $H_\sigma$ before $g_\sigma$ does, hence before $\tau$. If every seal agent of $\sigma$ made an extra transit upon leaving $H_\sigma$, these transits would give $\ell+1$ extra transits before $\tau$, contradicting $r \le \ell$.
Hence, some seal agent of $\sigma$ moves from $H_\sigma$ to $b^t_\sigma$ using a single transit through $c$.
From then on $b^t_\sigma$ is occupied, so no original agent, i.e., corresponding to an LSR object, can enter or leave $\sigma$.
Since this occurs before $\tau$, no original agent enters or leaves $\sigma$ after $\tau$.
Since this holds for every $\sigma$, no original agent transits through $c$ after $\tau$, so every extra transit of an original agent is among the $r \le \ell$ transits before $\tau$.
We can therefore obtain a corresponding LSR solution, which is extracted by ignoring seal, gate and padding agents, with at most $\ell$ relocations, as required.
\end{proof}

\section{Conclusion}
We have shown that stack-based motion is a simple source of hardness across various fundamental motion coordination problems, settling the decades-old question of the complexity of pebble motion on trees.
A key feature of the framework is that the same underlying reduction structure holds for distinct objectives, i.e., distance, makespan, and flowtime, which have typically required dedicated reductions or more involved gadget adaptations in previous work.
There is therefore potential for a reusable basis for further hardness results in highly restricted tree-like settings. %
At the same time, any positive results must now account for these basic hardness results, which can serve as a baseline obstruction that could inform, e.g., approximation results.

The results are tight along natural axes:
For the colored variants, hardness already appears with two colors, whereas the unlabeled motion is efficiently solvable~\cite{YuUnlabeledMAPF}.
Another boundary is established with respect to the maximum degree of the tree: our PMT reductions give hardness for maximum degree 3, while maximum degree 2 restricts the tree to a path, where the motion can be optimized efficiently. %

On the labeled side, the LSR construction uses stacks of depth 3, which is also the depth of the subdivided stars for distance and makespan.
We believe that the LSR construction can be tightened to stacks of depth 2; this refinement is left to a subsequent version.

Our analysis also identifies a tractable boundary case, which may be developed further to improve PMT algorithms.
By \Cref{lem:atomicity}, if a PMT solution attains the shortest-path lower bound, then its moves can be reordered so that each pebble performs its entire shortest-path traversal using consecutive moves. Consequently, for labeled PMT, one can test whether the lower bound is attainable by searching for an ordering of such atomic traversals.
This can be done in polynomial time since monotone motion along fixed paths reduces to cycle detection on a dependency graph~\cite{DBLP:conf/icra/Buckley89, DBLP:journals/comgeo/AbellanasBHORT06}.
This is also noteworthy because the analogous lower-bound attainability problem is NP-hard on 2D grids~\cite{GH2021AAMAS}.

A natural next step is to
identify the exact boundary between NP-hardness and polynomial-time solvability on trees.
In particular, can the hardness results be further
tightened by reducing parameters such as depth and the maximum degree? %
Are there efficient algorithms for special cases?

\bibliography{MRMP_references}

\end{document}